\documentclass[a4paper,12pt]{article}

\usepackage{lipsum}
\usepackage{amsfonts}
\usepackage{graphicx}
\usepackage{epstopdf}
\usepackage{algorithmic}
\usepackage{relsize}
\usepackage{tikz}

\usepackage{tabularx}
\usepackage{booktabs}

\usetikzlibrary{positioning}
\ifpdf
  \DeclareGraphicsExtensions{.eps,.pdf,.png,.jpg}
\else
  \DeclareGraphicsExtensions{.eps}
\fi

\usepackage{enumitem}
\usepackage{tikz}
\usepackage{tkz-euclide}
\setlist[itemize]{leftmargin=.5in}

\usepackage{amsmath,amsthm}
\usepackage{amssymb}
\usepackage{hyperref}
\usepackage{algorithm}
\usepackage{geometry}
\usepackage{esint}
\usepackage{array}
\usepackage[linewidth=2pt]{mdframed}
\newcounter{taggedeq}
\pretocmd{\equation}{\stepcounter{taggedeq}}{}{}

\newmdenv[
topline=false,
bottomline=false,
rightline=false,
skipabove=\topsep,
skipbelow=\topsep,
linewidth=4
]{siderules}
\usepackage{xcolor}
\usepackage[most]{tcolorbox}
\usepackage[nottoc]{tocbibind}

\newcommand{\requ}{\text{ReQU}}
\newcommand{\repu}{\text{RePU}}

\newtheorem{theorem}{Theorem}

\newtheorem{remark}[theorem]{Remark}
\newtheorem{definition}[theorem]{Definition}
\newtheorem{lemma}[theorem]{Lemma}
\newtheorem{corollary}[theorem]{Corollary}

\newtheorem*{proposition*}{Proposition}

\newcommand*\dx{\mathop{}\!\mathrm{d}}
\DeclareMathOperator{\sgn}{sgn}

\title{$\mathcal{C}^1$-approximation with rational functions and rational neural networks}
\author{Martin Holler \thanks{IDea\_Lab - The Interdisciplinary Digital Lab at the University of Graz. MH further is a member of NAWI Graz (\href{https://www.nawigraz.at}{www.nawigraz.at}) and of BioTechMed Graz (\href{https://biotechmedgraz.at}{biotechmedgraz.at}) (\href{mailto:martin.holler@uni-graz.at}{martin.holler@uni-graz.at})} \and Erion Morina \thanks{Corresponding author. IDea\_Lab - The Interdisciplinary Digital Lab at the University of Graz.
    (\href{mailto:erion.morina@uni-graz.at}{erion.morina@uni-graz.at}).}}

\usepackage{amsopn}

\ifpdf
\hypersetup{
	pdftitle={$\mathcal{C}^1$-approximation with rational functions and rational neural networks},
	pdfauthor={M. Holler and E. Morina}
}
\fi

\begin{document}	
	\maketitle
	\begin{abstract}
	We show that suitably regular functions can be approximated in the $\mathcal{C}^1$-norm both with rational functions and rational neural networks, including approximation rates with respect to width and depth of the network, and degree of the rational functions. As consequence of our results, we further obtain $\mathcal{C}^1$-approximation results for rational neural networks with the $\text{EQL}^\div$ and ParFam architecture, both of which are important in particular in the context of symbolic regression for physical law learning.
	\end{abstract}
	\section{Introduction}
	In the past decades, intensive efforts have been made to study neural networks, which are a fundamental building block of modern machine learning algorithms. One reason for their popularity lies in their universal approximation property, which most architectures share to some extent. This means that, for sufficiently large networks, it is possible to achieve an arbitrarily good approximation of appropriate functions in a suitable norm. The works of \cite{devore21, elbraecther21, gribonval20} provide an excellent overview of this aspect. Apart from the fact that approximation quality can be achieved by adjusting the network size, there is a strong awareness in the community that the choice of activation function plays an essential role in the network design when it comes to achievable practical performance. This affects vanishing gradients for saturated activation functions or, conversely, exploding gradients for higher-order polynomial activation functions \cite{bengio94}. 
The choice of activation function is also becoming increasingly important for approximation result: While classical approximation results such as \cite{mhaskar} hold for a large class of activation functions, works such as
\cite{belomestny23} exploit advantages of particular activation functions, in this case \textit{Rectified Quadratic Units} (ReQU), to achieve improved approximation results.

	An important class of activation functions that has recently attracted considerable interest are rational functions, which, in short, are fractions of polynomials with positive denominator polynomials. A pivotal work in this direction comes from Telgarksy \cite{telgarsky17}, in which the approximation of ReLU neural networks by high-degree rational functions and vice versa is proven with error bounds that are polynomial-logarithmic with respect to the degree of the approximating rational function and the size of the approximating network, respectively. The work \cite{Boulle20} shows similar results with optimal error bounds. The advantage of the approach proposed in \cite{Boulle20} is that the error bounds can be achieved with a comparatively smaller depth of the rational neural network by composing low-degree rational functions and thus, leveraging the approximation power of high-degree rational functions. It is important to note that this is possible due to the specific choice of composed rational functions, which are \textit{Zolotarev} sign functions (see \cite{Freund16}). In contrast, the work in \cite{telgarsky17} uses \textit{Newman} polynomials (see \cite{Newman}). We also refer to \cite{trefethen24} for the computation of Zolotarev rational functions.
	
	In a broader context, classical approximation theory using  rational functions has a long history, which is comprehensively presented, for example, in the work \cite{Petrushev_Popov_1988} with connections to spline approximation (e.g. succinctly discussed in \cite{Williamson91}). In fact, the decisive step in establishing the connection between spline and rational approximation is the best approximation of the absolute value function \cite{Newman} and monomials \cite{trefethen18}, which can be further used to represent splines.
	
	From a numerical point of view, a powerful approach for approximation by rational functions is the \textit{AAA Algorithm} proposed in \cite{nakatsukasa18} for a real or complex set of points and the generalization \cite{Nakatsukasa24} on a continuum with adaptive discretization, both of which are based on the scheme in \cite{Antoulas86}. The AAA Algorithm is used in practice in signal processing, for example (see, for instance, \cite{plonka22} and the references therein).
	
	An important question that, to our knowledge, has not yet been addressed in the literature is that of $\mathcal{C}^1$-approximation (i.e., uniform approximation of both the function and its derivative) of neural networks and general, sufficiently regular functions by rational functions. This is in contrast to the question of approximation rates (which has been well-studied in the setup outlined above by \cite{Boulle20}) and regularity of the target functions (for which there are also general classical results \cite{Petrushev_Popov_1988}).
		
Note that this question is not trivial since, as opposed to the case of polynomials which are closed with respect to integration, one cannot approximate the derivative of the target function using rationals and consider its antiderivative as an approximation of the original function. Instead, our strategy here is to $\mathcal{C}^1$-approximate activation functions of certain neural networks with rational functions, and then get $\mathcal{C}^1$-approximability of general functions as consequence of this and $\mathcal{C}^1$-approximation properties of the considered neural networks. Specifically, we will consider ReQU-based neural networks for this purpose, for which \cite{belomestny23} provides approximation results with with respect to higher-order Hölder norms.

	Application-wise, we note that $\mathcal{C}^1$-approximability of sufficiently regular functions by certain neural networks  is not only interesting in itself and due to the aspects discussed above, but also has important applications in optimal control and model learning. In \cite{kunisch21}, a $\mathcal{C}^1$-approximability result is used to show well-posedness and convergence of neural network based approximations to the stabilization problem of identifying the underlying optimal feedback law. Such a result is also relevant in \cite{dong22} for the convergence of solutions of approximating learning-informed PDEs to the solution of the limit PDE. Yet another work in the field of optimal control is \cite{Dong24} on smoothing networks for ReLU-neural-network-informed PDEs. In \cite{morina24}, which deals with uniqueness in learning PDE based physical laws, a first order approximability condition is required to show unique reconstructability of unknowns in the limit problem and approximability based on incomplete, noisy measurements.
	
As consequence of our results we will also obtain $\mathcal{C}^1$-approximation results for the ParFam architecture of \cite{scholl24} and for the $\text{EQL}^\div$ of \cite{sahoo18}, both proposed in the context of symbolic regression for physical law learning, with the $\text{EQL}^\div$ being heavily used in practical application.

	\paragraph{Contributions.} Our main contribution is the $\mathcal{C}^1$-approximation of certain spline-based ReQU neural networks, which attain a higher order universal approximation property, by rational neural networks. We then transfer this to the interesting result that suitably regular functions can be approximated uniformly up to first order derivative by rational functions at a rate that is better than under polynomial approximation. Besides that we provide a higher-order approximation result for rational function based neural networks used in symbolic regression for physical law learning.
	
	\paragraph{Structure of the paper.} In Section \ref{sec:aux} we prove a preliminary approximation result for the ReQU using rational functions. In Subsection \ref{subsec:approx} we prove the main results of $\mathcal{C}^1$-approximability by rational neural networks and rational functions. Finally, we consider applications to symbolic regression in Subsection \ref{subsec:symbolic}.
	\section{Auxiliary results}
	\label{sec:aux}
	In this section, we formulate and prove a preliminary approximation result, which will be needed later for the main result and applications. For this preliminary result, we use rational functions as approximating functions, which are defined as follows. Here, recall that two polynomials are called \textit{coprime} if they have no common nontrivial polynomial divisor, i.e., if there is no non-constant polynomial whose degree is less then the degrees of both polynomials and which divides both of them.
	\begin{definition}[Rational functions]
		\label{def:rat_func}
		We call a function $r:\mathbb{R}\to \mathbb{R}$ rational if there exist polynomials $p,q:\mathbb{R}\to\mathbb{R}$ with $q(x)>0$ for all $x\in\mathbb{R}$ and $p\neq 0$ such that $r(x) = \frac{p(x)}{q(x)}$ for all $x\in\mathbb{R}$. A rational function $r$ said to be of type $(n,m)$ if $r(x) = \frac{p(x)}{q(x)}$ for all $x\in\mathbb{R}$ with polynomials $p,q:\mathbb{R}\to\mathbb{R}$ that are coprime and of degree $n,m\in\mathbb{N}_0$, respectively. For a rational function $r$ of type $(n,m)$ we further call $\deg(r):=\max(n,m)$ the degree of $r$.
		 This notion is directly generalized to multidimensional rational functions as follows. For $d\in\mathbb{N}$ and $n\in\mathbb{N}_0$ consider a generic multidimensional polynomial $p:\mathbb{R}^d\to \mathbb{R}$ given by
		$$p(z)=\sum_{j_1,\dots, j_d=0}^na_{j_1,\dots,j_d}\prod_{l=1}^d z_l^{j_l}$$
		with $a_{j_1,\dots,j_d}\in\mathbb{R}$ for $0\leq j_1,\dots, j_d\leq n$. Then the degree of $p$ is given by
		\[
			\deg(p)=\max\left\{\sum_{l=1}^{d}j_l~:~a_{j_1,\dots,j_d}\neq 0\right\}.
		\]
		The degree of the rational function $r(z)=\frac{p(z)}{q(z)}$ for coprime polynomials $p,q:\mathbb{R}^d\to \mathbb{R}$ with $q(z)>0$ for $z\in\mathbb{R}^d$ and $p\neq 0$ is given by $\deg(r):=\max(\deg(p),\deg(q))$. Note that for any representation $r=\tilde p/\tilde q$ with $\tilde p,\tilde q$ polynomials that are not necessarily coprime, it holds that $\deg(r)\leq \max(\deg(\tilde p),\deg(\tilde q))$.
	\end{definition}
	\begin{remark}[Regularity]
		A rational function $r:\mathbb{R}^d\to\mathbb{R}$ attains the regularity $r\in\mathcal{C}^\infty(\Omega)$ for any bounded domain $\Omega\subset \mathbb{R}^d$. This follows from the fact that for a representation of $r$ as $r=p/q$ with polynomials $p$ and $q>0$, we have for $\gamma\in\mathbb{N}_0^d$ $$D^\gamma r = q^{-(1+\vert \gamma\vert)}P_\gamma\in\mathcal{C}(\Omega)$$ for some polynomial $P_\gamma$ by iteration of the quotient rule of differentiation.
	\end{remark}
	\begin{remark}[Degree]
		\label{rem:degree}
		From Definition \ref{def:rat_func}, it immediately follows that for two rational functions $r_1, r_2:\mathbb{R}^d\to \mathbb{R}$, both $r_1\cdot r_2$ and $r_1+r_2$ are again rational functions. Furthermore, the degrees fulfill the inequalities
		\[
			\deg(r_1\cdot r_2)\leq \deg(r_1)+\deg(r_2)\quad \text{and} \quad \deg(r_1+r_2)\leq \deg(r_1)+\deg(r_2).
		\]
	\end{remark}
	We now prove that the Rectified Quadratic Unit (ReQU) defined by
	\begin{align*}
		\requ: \quad &\mathbb{R} \to \mathbb{R}\\
		& x\mapsto \max(x,0)^2
	\end{align*}
	can be approximated in $\mathcal{C}^1([-1,1])$ by rational functions $(R_n)_{n\in\mathbb{N}}$ of type $(n+1,n-1)$, respectively, at polynomial rate in $n$ of arbitrarily large order. The rational functions $(R_n)_n$ are based on \textit{Newman polynomials} as defined in \cite{Newman}. These polynomials attain the exponential-type roots $\left\{\exp(-in^{-1/2})~|~0\leq i\leq n-1\right\}$ in the interval $[0,1]$, for which we will need the following density property.
	\begin{lemma}
		\label{gridlemma}
		Let $n\in\mathbb{N}$ and define $\xi_n:=\exp(-n^{-1/2})$. Then for each $x\in[\xi_n^n,1]$ there exists a $0\leq j\leq n-1$ such that
		\[
			\vert x-\xi_n^j\vert \leq \xi_n^jn^{-1/2}\leq n^{-1/2}.
		\]
	\end{lemma}
	\begin{proof}
		Given $n\in\mathbb{N}$, for $0\leq i\leq n-1$ we note that $\xi_n^i = \exp(-in^{-1/2})$. Since the exponential function is convex and increasing, the mean value theorem implies
		\begin{equation}
			\label{meanval}
			\vert e^{-z}-e^{-y}\vert\leq e^{-y}\vert y-z\vert
		\end{equation}
		for $y\leq z$. By substituting $y=in^{-1/2}$ and $z=(i+1)n^{-1/2}$ into \eqref{meanval}, we derive
		\begin{equation}
			\label{tightness}
			\vert \xi_n^{i+1}-\xi_n^i\vert \leq \xi_n^i n^{-1/2}\leq n^{-1/2}.
		\end{equation}
		Now fix $x\in[\xi_n^n,1]$. Since $\xi_n^0=1$ and $\xi_n^{i+1}<\xi_n^i$ for $0\leq i\leq n-2$, it follows that there exists $0\leq j\leq n-1$ such that $x\in[\xi_n^{j+1},\xi_n^j]$. Employing \eqref{tightness}, we obtain
		\[
		\vert x- \xi_n^j\vert \leq \vert \xi_n^{j+1}-\xi_n^j\vert\leq \xi_n^jn^{-1/2},
		\]
		which concludes the assertion of the lemma.
	\end{proof}
	With this, we verify an approximation result for the function $x\mapsto x^2\sgn(x)$ with 
	\begin{align*}
		\sgn:~ \mathbb{R}\to\mathbb{R}, \quad x\mapsto\begin{cases}
			1, & \text{if} ~ x>0\\
			0, & \text{if} ~ x=0\\
			-1, & \text{if} ~ x<0
		\end{cases}
	\end{align*}
	the sign function, which we will then transfer to $\requ$ as claimed.
	\begin{lemma}
		\label{lemma:approx1}
		Let $f:\mathbb{R}\to\mathbb{R}, ~ x\mapsto x^2\sgn(x)$. Then there exist rational functions $(R_n)_{n}$ of type $(n+1,n-1)$ for $n\in\mathbb{N}$ such that for every $K\in\mathbb{N}$ with $K\geq 3$:
		\begin{equation*}
			\Vert f-R_n\Vert_{\mathcal{C}([-1,1])} = \mathcal{O}(e^{-\sqrt{n}})\quad \text{and}\quad \Vert f'-R_n'\Vert_{\mathcal{C}([-1,1])} = \mathcal{O}(n^{-(K-2)/2}) ~ \text{as}~ n\to \infty.
		\end{equation*}
	\end{lemma}
	\begin{proof}
		We define the rational functions $(R_n)_n$ based on \cite{Newman} as follows. Let $P_n$ be the $n$-th Newman polynomial defined by
		\[
			P_n(x):=\prod_{i=0}^{n-1}(x+\xi_n^i)
		\]
		with $\xi_n:=\exp(-n^{-1/2})$ for $n\in\mathbb{N}$. Furthermore, define the rational functions
		\begin{equation}
			\label{newman_rationals}
			r_n(x):=\frac{P_n(x)-P_n(-x)}{P_n(x)+P_n(-x)},
		\end{equation}
		which are well defined for $x\in[-1,1]$ since $P_n(0)=\xi_n^{n(n-1)/2}>0$, 
		\[
			\vert P_n(-x)\vert =\prod_{i=0}^{n-1}\vert -x+\xi_n^i\vert<\prod_{i=0}^{n-1}\vert x+\xi_n^i\vert=\prod_{i=0}^{n-1}(x+\xi_n^i)=P_n(x)
		\]
		for $x>0$ and similarly $\vert P_n(x)\vert < P_n(-x)$ for $x<0$. Let the rational functions $(R_n)_n$ be defined by $R_n(x)=x^2r_n(x)$ (which are of type $(n+1,n-1)$). The result in \cite{Newman} shows that for $n\in\mathbb{N}$
		\[
			\Vert \vert x\vert-x^{-1}R_n(x)\Vert_{\mathcal{C}([-1,1])}\leq 3 e^{-\sqrt{n}}.
		\]
		Using that $\vert x\vert = x\sgn(x)$ it immediately follows that for $n\in\mathbb{N}$
		\[
			\Vert f-R_n\Vert_{\mathcal{C}([-1,1])}\leq 3 e^{-\sqrt{n}}.
		\]
		It remains to show that $\Vert f'-R_n'\Vert_{\mathcal{C}([-1,1])} = \mathcal{O}(n^{-(K-2)/2}) ~ \text{as}~ n\to \infty$. As
		\[
			R_n'(x) = 2x r_n(x)+x^2r_n'(x)
		\]
		and the function $x\mapsto 2xr_n(x)$ converges uniformly on $[-1,1]$ to the function $x\mapsto 2\vert x\vert = f'(x)$ as $n\to \infty$ with rate $e^{-\sqrt{n}}$ (which decays faster than $n^{-(K-2)/2}$), it suffices to show that
		\begin{equation}
			\label{eq_remaining}
			\Vert x^2r_n'(x)\Vert_{\mathcal{C}([-1,1])}=\mathcal{O}(n^{-(K-2)/2})
		\end{equation}
		as $n\to \infty$ for $K\geq 3$. Employing the quotient rule of differentiation yields
		\[
			r_n'(x)=2\frac{P_n'(x)P_n(-x)+P_n'(-x)P_n(x)}{(P_n(x)+P_n(-x))^2}
		\]
		and further, using that by Leibniz' rule of differentiation
		\[
			P_n'(x) = \sum_{i=0}^{n-1}\prod_{\substack{0\leq j\leq n-1\\j\neq i}}(x+\xi_n^j)=P_n(x)\sum_{i=0}^{n-1}(x+\xi_n^i)^{-1},
		\]
		we derive that the derivative can be rewritten as
		\begin{align}
			\label{eq_rderiv}
			\notag r_n'(x)&=4\frac{P_n(x)P_n(-x)}{(P_n(x)+P_n(-x))^2}\sum_{i=0}^{n-1}\xi_n^i\frac{1}{\xi_n^{2i}-x^2}\\&=4\frac{1}{(P_n(x)+P_n(-x))^2}\sum_{i=0}^{n-1}\xi_n^i\prod_{\substack{0\leq j\leq n-1\\j\neq i}}(\xi_n^{2j}-x^2).
		\end{align}
		By symmetry, for \eqref{eq_remaining} to hold true, it suffices to show that
		\begin{equation}
			\label{left_asymp}
			\Vert x^2r_n'(x)\Vert_{\mathcal{C}([0,1])}=\mathcal{O}(n^{-(K-2)/2})
		\end{equation}
		as $n\to\infty$. For that, we first estimate the term in \eqref{eq_rderiv}. We argue that
		\begin{equation}
			\label{lowerbound}
			\frac{P_n(-x)}{P_n(x)}\geq -\frac{1}{\sqrt{n}}
		\end{equation}
		for $0\leq x\leq 1$. For $x\in[0,\xi_n^n]$, both $P_n(-x)>0$ and $P_n(x)>0$, implying in particular \eqref{lowerbound}. Otherwise for $x\in[\xi_n^n,1]$, note first that $\vert \xi_n^i-x\vert(\xi_n^i+x)^{-1}\leq 1$ for $0\leq i\leq n-1$, which implies
		\begin{equation}
			\label{ineq_chain1}
			\frac{P_n(-x)}{P_n(x)}=\prod_{i=0}^{n-1}\frac{\xi_n^i-x}{\xi_n^i+x}\geq-\prod_{i=0}^{n-1}\frac{\vert\xi_n^i-x\vert}{\xi_n^i+x} \geq -\frac{\vert\xi_n^j-x\vert}{\xi_n^j+x}
		\end{equation}
		for each $0\leq j\leq n-1$. For $0\leq j_0\leq n-1$ fulfilling the assertion of Lemma \ref{gridlemma}, we derive with \eqref{ineq_chain1} that
		\[
			\frac{P_n(-x)}{P_n(x)}\geq -\frac{\xi_n^{j_0}}{\xi_n^{j_0}+x}\frac{1}{\sqrt{n}}\geq -\frac{1}{\sqrt{n}},
		\]
		which yields \eqref{lowerbound} also for $x\in[\xi_n^n,1]$. We now proceed by showing \eqref{left_asymp}.
		
		Using \eqref{lowerbound} we can estimate
		\[
			(P_n(x)+P_n(-x))^{-2}\leq P_n(x)^{-2}\left(1-n^{-1/2}\right)^{-2},
		\]
		which together with \eqref{eq_rderiv} implies for $x\in[0,1]$ the upper bound
		\begin{equation}
			\label{termtoestimate}
			x^2\vert r_n'(x)\vert\leq 4x^2\left(1-n^{-1/2}\right)^{-2}\sum_{i=0}^{n-1}\xi_n^i\frac{1}{(\xi_n^i+x)^2}\prod_{\substack{0\leq j\leq n-1\\j\neq i}}\frac{\vert\xi_n^j-x\vert}{\xi_n^j+x}.
		\end{equation}
		For $x\in[0,\xi_n^n]$ we can further estimate this term, using $\vert \xi_n^i-x\vert(\xi_n^i+x)^{-1}\leq 1$, by
		\[
			x^2\vert r_n'(x)\vert\leq 4x^2\left(1-n^{-1/2}\right)^{-2}\sum_{i=0}^{n-1}\xi_n^i\frac{1}{(\xi_n^i+x)^2}
		\]
		and, due to the estimation $x\xi_n^i(\xi_n^i+x)^{-2}\leq 1$ for $0\leq i\leq n-1$, by
		\begin{equation}
			\label{asymptotics1}
			x^2\vert r_n'(x)\vert\leq 4n\left(1-n^{-1/2}\right)^{-2}x\leq  4n\left(1-n^{-1/2}\right)^{-2}\xi_n^n=4n\left(1-n^{-1/2}\right)^{-2}e^{-\sqrt{n}},
		\end{equation}
		which is independent of $x\in[0,\xi_n^n]$. This upper bound decays asymptotically faster than $n^{-(K-2)/2}$ as $n\to \infty$ for any $K\geq 3$ (as the exponential grows faster than any polynomial). It remains to estimate \eqref{termtoestimate} for $x\in[\xi_n^n,1]$.
		
		In this case, for $n\geq K$ there exists a $0\leq j_0\leq n-K$ such that $x\in[\xi_n^{j_0+K-1},\xi_n^{j_0}]$. Now again by $\vert \xi_n^j-x\vert(\xi_n^j+x)^{-1}\leq 1$ it holds true for $0\leq i\leq n-1$ that
		\begin{equation}
			\label{prod_estimate}
			\prod_{\substack{0\leq j\leq n-1\\j\neq i}}\frac{\vert\xi_n^j-x\vert}{\xi_n^j+x}\leq \prod_{\substack{j_0\leq j\leq j_0+K-1\\j\neq i}}\frac{\vert\xi_n^j-x\vert}{\xi_n^j+x}.
		\end{equation}
		Employing Lemma \ref{gridlemma} and monotonicity, we derive for $j_0\leq j\leq j_0+K-1$ that
		\[
			\vert\xi_n^j-x\vert\leq \vert \xi_n^{j_0+K-1}-\xi_n^{j_0}\vert\leq \sum_{j=j_0}^{j_0+K-2}\vert\xi_n^{j+1}-\xi_n^j\vert\leq n^{-1/2}\sum_{j=j_0}^{j_0+K-2}\xi_n^j\leq n^{-1/2}(K-1)\xi_n^{j_0}.
		\]
		Hence, we may estimate for $j_0\leq j\leq j_0+K-1$ the fraction
		\[
			\frac{\vert\xi_n^j-x\vert}{\xi_n^j+x}\leq n^{-1/2}(K-1)\frac{\xi_n^{j_0}}{\xi_n^{j_0+K-1}+x}\leq n^{-1/2}(K-1)\xi_n^{1-K}.
		\]
		As a consequence, the term in \eqref{prod_estimate} is bounded by
		\begin{equation}
			\label{upperbound2}
		\prod_{\substack{0\leq j\leq n-1\\j\neq i}}\frac{\vert\xi_n^j-x\vert}{\xi_n^j+x}\leq n^{-(K-1)/2}(K-1)^{K-1}e^{(K-1)^2/\sqrt{n}}.
		\end{equation}
		Note here that if $i\notin\left\{j_0,\dots, j_0+K-1\right\}$ we can skip a factor on the right hand side of \eqref{prod_estimate} and make the product larger, finally recovering \eqref{upperbound2}.
		
		Now combining estimation \eqref{upperbound2} together with \eqref{termtoestimate} under the estimation $x^2(\xi_n^i+x)^{-2}\leq 1$ for $0\leq i\leq n-1$, we obtain for $x\in[\xi_n^n,1]$ that
		\[
			x^2\vert r_n'(x)\vert \leq 4(1-n^{-1/2})^{-2}n^{-(K-1)/2}(K-1)^{K-1}e^{(K-1)^2/\sqrt{n}}\sum_{i=0}^{n-1}\xi_n^i.
		\]
		We can estimate the truncated geometric series by
		\[
			\sum_{i=0}^{n-1}\xi_n^i=\sum_{i=0}^{n-1}\exp(-n^{-1/2})^i=\frac{1-e^{-\sqrt{n}}}{1-e^{-1/\sqrt{n}}}\leq\frac{1}{1-e^{-1/\sqrt{n}}} = \mathcal{O}(\sqrt{n})
		\]
		as $n\to \infty$, where the last asymptotic inclusion follows by the series representation of the exponential. Thus, as
		\[
			\lim_{n\to \infty} (1-n^{-1/2})^{-2}=1\quad \text{and} \quad \lim_{n\to \infty}e^{(K-1)^2/\sqrt{n}} = 1
		\]
		we derive that independently of $x\in[\xi_n^n,1]$ it holds
		\[
			x^2\vert r_n'(x)\vert = \mathcal{O}(n^{-(K-2)/2})
		\]
		as $n\to \infty$, which together with \eqref{asymptotics1} concludes the claim in \eqref{left_asymp} and also 
		\[
			\Vert f'-R_n'\Vert_{\mathcal{C}([-1,1])} = \mathcal{O}(n^{-(K-2)/2}) ~ \text{as}~ n\to \infty,
		\]
	completing the proof.
	\end{proof}
	Our final approximation result for $\requ$ is now a direct consequence of this.
	\begin{corollary}
		\label{cor:requ} 
		Let $f=\requ$. Then there exist rational functions $(\tilde R_n)_{n}$ of type $(n+1,n-1)$ for $n\in\mathbb{N}$ such that for every $K\in\mathbb{N}$ with $K\geq 3$:
		\begin{equation*}
			\Vert f-\tilde R_n\Vert_{\mathcal{C}([-1,1])} = \mathcal{O}(e^{-\sqrt{n}})\quad \text{and}\quad \Vert f'-\tilde R_n'\Vert_{\mathcal{C}([-1,1])} = \mathcal{O}(n^{-(K-2)/2}) ~ \text{as}~ n\to \infty.
		\end{equation*}
	\end{corollary}
	\begin{proof}
		The rectified quadratic unit $\requ(x)=\max(x,0)^2$ can be rewritten by 
		\[
			\requ(x)=\frac{1}{2}(x^2+x^2\sgn(x)).
		\]
		As a consequence, for the rational functions $(R_n)_n$ fulfilling Lemma \ref{lemma:approx1}, the assertion follows for the rational functions $\tilde R_n(x)=\frac{1}{2}(x^2+R_n(x))$.
	\end{proof}
		\begin{remark}[Comparison to polynomials]
		The problem of approximating $\requ $ in $\mathcal{C}^1([-1,1])$ could also be considered in the space of polynomials instead of the space of rational functions. As noted in \cite{Newman} for the absolute value function, the best approximation rate of $2\text{ReLU}(x)$ in $\mathcal{C}([-1,1])$ for degree-$n$ polynomials is $n^{-1}$. It can be easily shown that the primitives of these polynomials approximate $\requ$ in $\mathcal{C}^1([-1,1])$ at a rate $n^{-1}$, which is significantly worse than the polynomial rate of arbitrary high order proven in Corollary \ref{cor:requ} for rational functions.
	\end{remark}
	\begin{remark}[Denominator polynomials]
		\label{remark:denominator}
		In the proof of Lemma \ref{lemma:approx1}, we argue that $P_n(x)+P_n(-x)$, the denominator polynomial of \eqref{newman_rationals}, is positive. In fact, given the definition of the Newman polynomials $P_n$, we can deduce the stronger property that $P_n(x)+P_n(-x)$ is a polynomial in $x$ consisting only of monomial terms of even degree with positive coefficients.
	\end{remark}
	\begin{remark}[Generalization] 
		The result of Lemma \ref{lemma:approx1} also holds for $\repu_p$, the Rectified Power Unit of order $p\in\mathbb{N}_{\geq 3}$, with $\repu_p(x)=\max(x,0)^p$ since 
		\[
			\repu_p(x) = \frac{1}{2}(x^p+x^p\sgn(x))\quad \text{for} ~ x\in \mathbb{R}.
		\]
		In this case, the approximating rational functions can be chosen of the form $\tilde{R}_n(x)=\frac{1}{2}(x^p+x^pr_n(x))$ with $r_n$ as in \eqref{newman_rationals}. We conjecture that the $\tilde{R}_n$ approximate $\repu_p$ in $\mathcal{C}^{p-1}([-1,1])$ for $p\in\mathbb{N}$. For $p=1$ this is proven in \cite{Newman}, and for $p=2$, this is the result of Corollary \ref{cor:requ}. For the case $p\geq 3$, based on Leibniz' rule, the conjecture can be reduced to showing that for any $l\in\mathbb{N}$, it holds
		\[
		x^{1+l}r_n^{(l)}(x)\to 0\quad \text{uniformly for} ~ x\in[-1,1] \quad \text{as} ~ n\to \infty.
		\]
		We will not go into further detail and will leave it as a conjecture, since we currently see no suitable applications of this result.
	\end{remark}

	\section{Main result}
	\label{sec:main}
	In this section, we provide our main approximation results and consider an application in \textit{Symbolic Regression}. We begin by examining $\mathcal{C}^1$ approximation properties of rational functions.
	\subsection{First order approximation by rational functions}
	\label{subsec:approx}
	We no consider to what extend the $\mathcal{C}^1([-1,1])$-approximation result of Corollary \ref{cor:requ} can be extended to a general class of suitably regular functions $f$. For this, we use the work \cite{belomestny23}, which provides a feed forward neural network architecture with ReQU activation that attains an universal approximation property as follows. For $H>0$, $d\in\mathbb{N}$, $\beta>2$, $s = \lfloor \beta\rfloor$, $\delta=\beta-s$, define the Hölder ball
	\[
		\mathcal{H}^\beta_p([0,1]^d,H):=\left\{f\in\mathcal{C}^s([0,1]^d,\mathbb{R}^p):~ \Vert f\Vert_{\mathcal{H}^\beta}:=\max(\Vert f\Vert_{\mathcal{C}^s},\max_{\vert \gamma\vert=s}\left[D^\gamma f\right]_\delta)\leq H\right\}
	\]
	with $\vert\gamma \vert = \sum_{i=1}^d\gamma_i$ and the Hölder constant 
	\[
		\left[g\right]_\delta :=\max_{1\leq i\leq p}\sup_{x\neq y\in[0,1]^d}\frac{\vert g_i(x)-g_i(y)\vert}{\min(1,\Vert x-y\Vert)^\delta}
	\]
	for $g:[0,1]^d\to \mathbb{R}^p$. Furthermore, we denote $f\in \mathcal{H}^\beta_p([0,1]^d)$ if there exists some $H>0$ with $f\in \mathcal{H}^\beta_p([0,1]^d,H)$. With this, it is shown in \cite{belomestny23} that for $f\in \mathcal{H}^\beta_p([0,1]^d,H)$ with $\beta>2$ there exist a constant $c>0$ and feed forward ReQU neural networks $(f_N)_{N\in\mathbb{N}}$ with width of order $N$ and constant depth such that
	\begin{align}
		\label{belo_estim}
		\Vert f-f_N\Vert_{\mathcal{H}^l}\leq cN^{(l-\beta)/d}
	\end{align}
	for $0\leq l\leq \lfloor\beta\rfloor$. The estimate \eqref{belo_estim} implies in particular
	\begin{align}
		\label{simplified_rate}
		\Vert f-f_N\Vert_{\mathcal{C}^1([0,1]^d,\mathbb{R}^p)}\leq cN^{-(\beta-1)/d}
	\end{align}
	for $f\in\mathcal{H}^\beta_p([0,1]^d,H)$ with $\beta>2$. The goal in this section is to derive such an approximation result for rational functions instead of ReQU neural networks, possibly with a different rate. Our strategy for achieving this is to show that the ReQU neural networks $(f_N)_{N\in\mathbb{N}}$ can be approximated by rational functions in $\mathcal{C}^1$, and transfer this approximation result to functions $f\in\mathcal{H}^\beta_p([0,1]^d,H)$ using the estimate in \eqref{simplified_rate}. For that, we first replace the ReQU activation functions in the representation of $f_N$ with rational functions that approximate $\requ$ in the $\mathcal{C}^1$-sense, using Corollary \ref{cor:requ}. Then we show that the resulting objects, which are \textit{Rational Neural Networks} as defined below, approximate $f_N$ in the $\mathcal{C}^1$-sense and indeed define suitable rational functions. Using \eqref{simplified_rate} finally concludes the assertion.
	
First we define rational neural networks, which are feedforward neural networks whose activation functions are rational functions with the numerator/denominator polynomial coefficients being trainable parameters.
	\begin{definition}[Rational neural network]
		Let  $L\in\mathbb{N}$, $(n_l)_{0\leq l\leq L}\subseteq \mathbb{N}$ and $d^{ij}\in\mathbb{R}$ for $1\leq j\leq n_i$, $1\leq i\leq L$. Let $r^{ij}:\mathbb{R}^{n_{i-1}}\to \mathbb{R}$ be rational functions of degree $d^{ij}\in\mathbb{N}_0$ for $1\leq j\leq n_i$, $1\leq i\leq L$ as in Definition \ref{def:rat_func}. That is, there exist polynomials $p^{ij}\neq 0$, $q^{ij}>0$ with $\max(\deg(p^{ij}), \deg(q^{ij}))=d^{ij}$, parameterized by $\eta_{k_1^{ij}, \dots,k^{ij}_{n_{i-1}}}$,$\theta_{k_1^{ij}, \dots,k^{ij}_{n_{i-1}}}\in \mathbb{R}$ for $0\leq k_l^{ij}\leq d^{ij}$, $1\leq l\leq n_{i-1}$, $1\leq j\leq n_i$, $1\leq i\leq L$, 
		\[
			\text{i.e.}, ~ ~ p^{ij}(z)=\sum_{k_1^{ij},\dots, k_{n_{i-1}}^{ij}=0}^{d^{ij}}\eta_{k_1^{ij}, \dots,k^{ij}_{n_{i-1}}}\prod_{l=1}^{n_{i-1}}z_l^{k_l^{ij}}, ~ ~ q^{ij}(z)=\sum_{k_1^{ij},\dots, k_{n_{i-1}}^{ij}=0}^{d^{ij}}\theta_{k_1^{ij}, \dots,k^{ij}_{n_{i-1}}}\prod_{l=1}^{n_{i-1}}z_l^{k_l^{ij}}
		\]
		such that $r^{ij}=p^{ij}/q^{ij}$. Let $\mathcal{R}_i:\mathbb{R}^{n_{i-1}}\to\mathbb{R}^{n_i}$ via $\mathcal{R}_i(z)=(r^{ij}(z))_{j=1}^{n_i}$ for $1\leq i\leq L$. Then a rational neural network $\mathcal{R}$ is defined as $\mathcal{R}=\mathcal{R}_L\circ\dots\circ \mathcal{R}_1$. The input dimension of $\mathcal{R}$ is $n_0$ and the output dimension $n_L$. Moreover, we define the width of the rational neural network $\mathcal{R}$ by $\max_l n_l$ and the depth by $L$.
	\end{definition}
	Our main approximation result for rational neural network now reads as follows.
	\begin{theorem}
		\label{applic_theorem}
		Let $\beta>2$ and $p,d\in\mathbb{N}$. For any $f:[0,1]^d\to\mathbb{R}^p$ with $f\in\mathcal{H}^\beta_p([0,1]^d,H)$ for some $H>0$ and $\epsilon>0$ there exist a constant $c>0$ and rational neural networks $\mathcal{R}_f^{\lfloor\beta\rfloor,N}$ with width of order $N$, constant depth and maximal degree of the rational activation functions of order $N^\epsilon$ for $N\in\mathbb{N}$ such that
		\[
			\Vert f-\mathcal{R}_f^{\lfloor\beta\rfloor,N}\Vert_{\mathcal{C}^1([0,1]^d)}\leq cN^{-(\beta-1)}.
		\]
	\end{theorem}
	\begin{proof}
		The proof is structured into three parts: First, we recall the tensor-product-splines setting of \cite{belomestny23}. On this basis, we design the approximating rational neural network, and finally we prove the claimed approximation rate.
		
		\paragraph{Tensor-product splines in \cite{belomestny23}.}
		Following the proof of \cite[Theorem 2]{belomestny23} for $l=1$ and estimating the corresponding Hölder-norm accordingly, there exist a constant $c>0$ and tensor-product splines (which are multivariate splines) denoted by
		\begin{align}
			\label{estimSf}
			S_f^{\lfloor\beta\rfloor, N} = (S_{f,1}^{\lfloor\beta\rfloor, N}, \dots, S_{f,p}^{\lfloor\beta\rfloor, N})
		\end{align}
		of order $\lfloor \beta\rfloor$ (see \cite[Theorem 3]{belomestny23}) with knots at
		\[
			\left\{(a_{j_1},\dots, a_{j_d}):\quad j_1,\dots,j_d\in\left\{1,\dots,2\lfloor\beta\rfloor+N+1\right\}\right\}
		\]
		where $a_1=\dots=a_{\lfloor\beta\rfloor+1}=0$, $a_{\lfloor\beta\rfloor+j+1}=\frac{j}{N}$ for $1\leq j\leq N-1$ and $a_{\lfloor\beta\rfloor+N+1}=\dots=a_{2\lfloor\beta\rfloor+N+1}=1$ such that
		\begin{align}
			\label{Sfestimation}
			\Vert f-S_f^{\lfloor\beta\rfloor,N}\Vert_{\mathcal{C}^1([0,1]^d,\mathbb{R}^p)}\leq \max_{1\leq m\leq p}\Vert f_m-S_{f,m}^{\lfloor\beta\rfloor,N}\Vert_{\mathcal{C}^1([0,1]^d,\mathbb{R}^p)}\leq cN^{-(\beta-1)}.
		\end{align}
		There exist suitable coefficients $\left\{w_{m,j_1,\dots, j_d}^{(f)}: 1\leq m\leq p, 1\leq j_1,\dots, j_d\leq \lfloor\beta\rfloor+N\right\}$ for $1\leq m\leq p$ such that the spline $S_{f,m}^{\lfloor\beta\rfloor, N}$ attains the representation 
		\begin{align}
			\label{S_representation}
			S_{f,m}^{\lfloor\beta\rfloor,N}(x)=\sum_{j_1,\dots,j_d=1}^{\lfloor\beta\rfloor+N}w_{m,j_1,\dots, j_d}^{(f)}\prod_{l=1}^d(a_{j_l+\lfloor\beta\rfloor+1}-a_{j_l})B_{j_l}^{\lfloor\beta\rfloor,N}(x_l)
		\end{align}
		for $x\in[0,1]^d$ (see \cite[Equation (7)]{belomestny23}) with unnormalized $B$-splines $B_{j_l}^{\lfloor\beta\rfloor,N}$ (which are univariate basis splines, defined recursively in \cite[Equation (A.1)]{belomestny23}). To prove the claimed assertion, we consider the approximation of $S_f^{\lfloor\beta\rfloor,N}$ in $\mathcal{C}^1([0,1]^d,\mathbb{R}^p)$ by rational neural networks. As $S_{f,m}^{\lfloor\beta\rfloor,N}$ is a polynomial (and hence, rational) transformation of type $(d,0)$ or degree $d$ of the $B$-splines $B_{j_l}^{\lfloor\beta\rfloor,N}$ due to \eqref{S_representation}, this breaks down to approximating the $B_{j_l}^{\lfloor\beta\rfloor,N}$ by rational neural networks in $\mathcal{C}^1([0,1]^d)$.
		
		Following the proof of \cite[Lemma 3]{belomestny23} it suffices to consider the rational approximation of the $B$-splines $B_j^{2,N}$ for $1\leq j\leq N+2\lfloor\beta\rfloor-2$. The reason is that due to \cite[(A.1)]{belomestny23} for $m\in\left\{3,\dots,\lfloor\beta\rfloor\right\}$, $1\leq j\leq 2\lfloor\beta\rfloor+N-m$ the recursion
		\begin{align}
			\label{recursion}
			B_j^{m,N}(z) = (a_{j+m+1}-a_j)^{-1}\left((z-a_j)B_j^{m-1,N}(z)+(a_{j+m+1}-z)B_{j+1}^{m-1,N}(z)\right)
		\end{align}
		if $a_j<a_{j+m+1}$ and $B_j^{m,N}(z)=0$ otherwise, applies for $z\in[0,1]$. Hence, the $B$-spline $B_j^{m,N}$ is a rational transformation of type $(2,0)$ or degree $2$ of $z$ and the $B$-splines $B_j^{m-1,N}$. 

		Following the proof of \cite[Lemma 2]{belomestny23}, the following representations hold for the $B$-splines $B_j^{2,N}$ with $1\leq j\leq 2q+N-2$. For $z\in[0,1]$ with $q=\lfloor\beta\rfloor$ it holds
		\begin{multline*}
			B_j^{2,N}(z) = \frac{N^3}{6}\bigg(\requ\left(z-\frac{j-q-1}{N}\right)-3\requ\left(z-\frac{j-q}{N}\right)\\
			+3\requ\left(z-\frac{j-q+1}{N}\right)-\requ\left(z-\frac{j-q+2}{N}\right)\bigg)
		\end{multline*}
		for $q+1\leq j\leq q+N-2$. For $1\leq j\leq q-2$ and $q+N-1\leq j\leq 2q+N-2$ we have $B_j^{2,N}(z)=0$. The remaining splines attain the representations
		\begin{align*}
			B_{q-1}^{2,N}(z) &= N^3\requ(\frac{1}{N}-z),\\
			B_q^{2,N}(z) &= \frac{N^3}{4}\left(\requ\left(\frac{2}{N}-z\right)-4\requ\left(\frac{1}{N}-z\right)+3\requ(-z)\right),\\
			B_{q+N-1}^{2,N}(z) &= \frac{N^3}{4}\left(\requ\left(z-\frac{N-2}{N}\right)-2\requ\left(z-\frac{N-1}{N}\right)-3\requ(z-1)\right),\\
			B_{q+N}^{2,N}(z) &= N^3\requ\left(z-\frac{N-1}{N}\right).
		\end{align*}
		With this, we now design the rational neural network that approximates the tensor product spline $S_f^{\lfloor\beta\rfloor, N}$.
		
		\paragraph{Design of rational neural network.}
		As a consequence of the previous considerations, the only terms that need to be approximated by rational functions are the ReQU terms in the representation of the $B$-splines $B_j^{2,N}$. For that, we employ Corollary \ref{cor:requ}, which implies the existence of rational functions $(R_M)_M$ of type $(M+1,M-1)$ (degree $M+1$) such that for any fixed $K\in \mathbb{N}$ it holds
		\begin{align}
			\label{rate}
			\Vert \requ-R_M\Vert_{\mathcal{C}^1([-1,1])}\leq CM^{-K}
		\end{align}
		as $M\to \infty$ for a constant $C>0$. Using this, we infer for $0\leq \alpha\leq 1$ both that
		\begin{equation}
			\label{gen_asymp}
			\begin{aligned}
				&\Vert \requ(\cdot-\alpha)-R_M(\cdot-\alpha)\Vert_{\mathcal{C}^1([0,1])}\leq CM^{-K}\quad\text{and}\\
				&\Vert \requ(\alpha-\cdot)-R_M(\alpha-\cdot)\Vert_{\mathcal{C}^1([0,1])}\leq CM^{-K}.
			\end{aligned}
		\end{equation}
		Consequently, rational approximations in $\mathcal{C}^1([0,1])$ of the above-mentioned $B$-splines of order $2$ can be obtained by simply replacing the ReQU terms with the rational functions $R_M$. We are now ready to propose the rational neural network design for approximating $S_f^{\lfloor\beta\rfloor, N}$, which we will do layer by layer. 
		
		\textit{Input layer.} The input layer (with input $x\in [0,1]^d$) is assigned the index $0$. The network will consist of $\lfloor\beta\rfloor+1$ intermediate layers.
		
		\textit{Layer $1$.} Each of the $x_i$ for $1\leq i\leq d$ is mapped to $x_i-j/N$ for $0\leq j\leq N$, respectively, resulting in a layer of width $(N+1)d$. The rational functions that achieve this are all of degree $1$ (affine linear functions).
		
		\textit{Layer $2$.} In this layer the terms $R_M(x_i-j/N)$, $R_M(-x_i)$,$R_M(1/N-x_i)$, $R_M(2/N-x_i)$ are generated for $1\leq i\leq d$, $0\leq j\leq N$, which are exactly the ReQU terms occurring in representation of the $B$-splines of order $2$. The rational function that achieves this is $R_M$ with degree $M+1$. We also transfer the $x_i$ for $1\leq i\leq d$ from the previous layer via the identity (which is rational with degree $1$). The width of this layer is $(N+5)d$.
		
		\textit{Layer $3$.} In this layer the terms $B_{j}^{2,N}(x_i)$ are generated for $1\leq i\leq d$ and $\lfloor\beta\rfloor-1\leq j\leq \lfloor\beta\rfloor+N$. This can be achieved with rational functions of degree $1$. Furthermore, we transfer $x$ from the previous layer, resulting in an overall width of $(N+3)d$.
		
		\textit{Layer $m+1$ for $3\leq m\leq\lfloor\beta\rfloor$ .} In this layer the terms $B_{j}^{m,N}(x_i)$ are generated for $1\leq i\leq d$ and $\lfloor\beta\rfloor-m+1\leq j\leq \lfloor\beta\rfloor+N$. This can be achieved with rational functions of degree $2$ in view of the recursion \eqref{recursion}. Furthermore, we transfer $x$ from the previous layer (except for $m=\lfloor\beta\rfloor$), resulting in an overall width of $(N+m+1)d$.
		
		\textit{Output layer.} In the final layer the value of $S_{f}^{\lfloor\beta\rfloor,N}(x)$ is calculated employing the formula \eqref{S_representation} on layer $\lfloor\beta\rfloor+1$, which requires the application of a degree $d$ polynomial (and hence, also degree $d$ rational). The output has the $p$ entries $S_{f,l}^{\lfloor\beta\rfloor,N}(x)$ for $1\leq l\leq p$.
		
		We denote the resulting rational neural network by $\mathcal{R}_{f}^{\lfloor\beta\rfloor,N,M}$ of depth $\lfloor\beta\rfloor+2$, width $(N+\max(4,\lfloor\beta\rfloor)+1)d$ and maximal degree $M+1$ of rational functions.
		
		\paragraph{Approximation rate.} We denote by $\tilde{B}_j^{2,N,M}$ the functions obtained by changing the ReQU terms in the definition of the $B$-splines of order $2$ with the rational function $R_M$. Furthermore, we denote by $\tilde{B}_j^{m,N,M}$ for $3\leq m\leq \lfloor\beta\rfloor$ the functions derived by applying the recursion formula \eqref{recursion} to $\tilde{B}^{m-1,N,M}_j$.
		
		We start by analyzing the approximation error 
		\[
			\Vert B_{j}^{m,N}-\tilde{B}_{j}^{m,N,M}\Vert_{\mathcal{C}^1([0,1])}
		\]
		for $2\leq m\leq \lfloor\beta\rfloor$ and $\lfloor\beta\rfloor-m+1\leq j\leq \lfloor\beta\rfloor+N$.
		
		For $m=2$ it holds by the choice of $(R_M)_M$ outlined above that
		\begin{align}
			\label{2estim}
			\Vert B_{j}^{2,N}-\tilde{B}_{j}^{2,N,M}\Vert_{\mathcal{C}^1([0,1])}\leq 2CN^3M^{-K}.
		\end{align}
		For $3\leq m\leq \lfloor\beta\rfloor$ it holds due to recursion \eqref{recursion} and the definition of $\tilde{B}_j^{m,N,M}$ that
		\begin{multline*}
			\Vert B_{j}^{m,N}-\tilde{B}_{j}^{m,N,M}\Vert_{\mathcal{C}([0,1])}\leq \max_{a_j\leq z\leq a_{j+m+1}}\frac{\vert z-a_j\vert}{a_{j+m+1}-a_j}\Vert B_j^{m-1,N}-\tilde{B}_j^{m-1,N,M}\Vert_{\mathcal{C}([0,1])}\\+\max_{a_j\leq z\leq a_{j+m+1}}\frac{\vert a_{j+m+1}-z\vert}{a_{j+m+1}-a_j}\Vert B_{j+1}^{m-1,N}-\tilde{B}_{j+1}^{m-1,N,M}\Vert_{\mathcal{C}([0,1])}\\
			\leq 2\max_l\Vert B_l^{m-1,N}-\tilde{B}_l^{m-1,N,M}\Vert_{\mathcal{C}([0,1])}.
		\end{multline*}
		Similarly, for the derivatives it holds true that
		\begin{multline*}
			\Vert \frac{\dx}{\dx z}(B_{j}^{m,N}-\tilde{B}_{j}^{m,N,M})\Vert_{\mathcal{C}([0,1])}\\
			\leq \frac{N}{m+1}\Vert B_j^{m-1,N}-\tilde{B}_j^{m-1,N,M}\Vert_{\mathcal{C}([0,1])}
			+\frac{N}{m+1}\Vert B_{j+1}^{m-1,N}-\tilde{B}_{j+1}^{m-1,N,M}\Vert_{\mathcal{C}([0,1])}\\
			+\Vert \frac{\dx}{\dx z}(B_j^{m-1,N}-\tilde{B}_j^{m-1,N,M})\Vert_{\mathcal{C}([0,1])}
			+\Vert \frac{\dx}{\dx z}(B_{j+1}^{m-1,N}-\tilde{B}_{j+1}^{m-1,N,M})\Vert_{\mathcal{C}([0,1])}
		\end{multline*}
		and overall that
		\[
			\Vert B_{j}^{m,N}-\tilde{B}_{j}^{m,N,M}\Vert_{\mathcal{C}^1([0,1])}\leq 2N\max_l\Vert B_l^{m-1,N}-\tilde{B}_l^{m-1,N,M}\Vert_{\mathcal{C}^1([0,1])}.
		\]
		Resolving the recursion yields
		\[
			\Vert B_{j}^{\lfloor\beta\rfloor,N}-\tilde{B}_{j}^{\lfloor\beta\rfloor,N,M}\Vert_{\mathcal{C}^1([0,1])}\leq(2N)^{\lfloor\beta\rfloor-2}\max_l\Vert B_l^{2,N}-\tilde{B}_l^{2,N,M}\Vert_{\mathcal{C}^1([0,1])}
		\]
		and employing \eqref{2estim}, finally,
		\begin{align}
			\label{betaestimation}
			\Vert B_{j}^{\lfloor\beta\rfloor,N}-\tilde{B}_{j}^{\lfloor\beta\rfloor,N,M}\Vert_{\mathcal{C}^1([0,1])}\leq 2^{\lfloor\beta\rfloor-1}CN^{\lfloor\beta\rfloor+1}M^{-K}.
		\end{align}
		The next step consists in estimating the approximation error of the products
		\begin{align}
			\label{product_estim}
			\bigg\Vert\prod_{l=1}^dB_{j_l}^{\lfloor\beta\rfloor, N}-\prod_{l=1}^d\tilde{B}_{j_l}^{\lfloor\beta\rfloor, N,M}\bigg\Vert_{\mathcal{C}^1([0,1])}.
		\end{align}
		For that, we will need estimates of $\Vert B_j^{\lfloor\beta\rfloor,N}\Vert_{\mathcal{C}([0,1])}$ and $\Vert \tilde{B}_j^{\lfloor\beta\rfloor,N}\Vert_{\mathcal{C}([0,1])}$. Employing the recursion formula \eqref{recursion} once more yields
		\begin{align}
			\label{bestim}
			\Vert B_j^{\lfloor\beta\rfloor,N}\Vert_{\mathcal{C}([0,1])}\leq 2\Vert B_j^{\lfloor\beta\rfloor-1,N}\Vert_{\mathcal{C}([0,1])}\leq 2^{\lfloor\beta\rfloor-2}\Vert B_j^{2,N}\Vert_{\mathcal{C}([0,1])}\leq 2^{\lfloor\beta\rfloor-2}N,
		\end{align}
		where the last estimation is a consequence of $\Vert B_j^{2,N}\Vert_{\mathcal{C}([0,1])}\leq N$, which follows from direct but tedious computations. Due to \eqref{betaestimation} and the triangle inequality we directly derive that
		\begin{align}
			\label{btildeestim}
			\Vert \tilde{B}_j^{\lfloor\beta\rfloor,N}\Vert_{\mathcal{C}([0,1])}\leq 2^{\lfloor\beta\rfloor-1}CN^{\lfloor\beta\rfloor+1}M^{-K}+2^{\lfloor\beta\rfloor-2}N.
		\end{align}
		In view of \eqref{product_estim}, a telescope argument and the estimations \eqref{betaestimation},\eqref{bestim} and \eqref{btildeestim} give
		\begin{multline}
			\label{longestimate}
			\bigg\Vert\prod_{l=1}^dB_{j_l}^{\lfloor\beta\rfloor, N}-\prod_{l=1}^d\tilde{B}_{j_l}^{\lfloor\beta\rfloor, N,M}\bigg\Vert_{\mathcal{C}([0,1])}\\ \leq \bigg\Vert\sum_{l=1}^d(B_{j_l}^{\lfloor\beta\rfloor,N}-\tilde{B}_{j_l}^{\lfloor\beta\rfloor,N,M})\prod_{k<l}B_{j_k}^{\lfloor\beta\rfloor, N}\prod_{i>l}\tilde{B}_{j_i}^{\lfloor\beta\rfloor, N,M}\bigg\Vert_{\mathcal{C}([0,1])}\\
			\leq 2^{\lfloor\beta\rfloor-1}dCN^{\lfloor\beta\rfloor+1}M^{-K}(2^{\lfloor\beta\rfloor-1}CN^{\lfloor\beta\rfloor+1}M^{-K}+2^{\lfloor\beta\rfloor-2}N)^{d-1}\\
			\leq \tilde{C}N^{\lfloor\beta\rfloor+1}M^{-K}(N^{\lfloor\beta\rfloor+1}M^{-K}+N)^{d-1}
		\end{multline}
		for some $\tilde{C}>0$. It remains to estimate the first order part in \eqref{product_estim}. For that, note that by Leibniz' rule of differentiation it holds true that
		\begin{align}
			\label{product_derivative}
			\frac{\dx}{\dx z}\prod_{l=1}^dB_{j_l}^{\lfloor\beta\rfloor, N} = \sum_{l=1}^d\frac{\dx}{\dx z}B_{j_l}^{\lfloor\beta\rfloor,N}\prod_{k\neq l}B_{j_k}^{\lfloor\beta\rfloor,N}.
		\end{align}
		Furthermore, we derive by \eqref{recursion} that
		\begin{multline*}
			\frac{\dx}{\dx z}B_j^{m,N}(z)=(a_{j+m+1}-a_j)^{-1}(B_j^{m-1,N}(z)-B_{j+1}^{m-1,N}(z))\\
			+z(a_{j+m+1}-a_j)^{-1}(\frac{\dx}{\dx z}B_j^{m-1,N}(z)-\frac{\dx}{\dx z}B_{j+1}^{m-1,N}(z))
		\end{multline*}
		and as a consequence of the definition of the $a_j$ for $z\in[0,1]$ that 
		\[
			\max_j\Vert\frac{\dx}{\dx z}B_j^{m,N}\Vert_{\mathcal{C}([0,1])}\leq 2N\max_j\Vert B_j^{m-1,N}\Vert_{\mathcal{C}^1([0,1])}.
		\]
		Using similar arguments as before, this implies that
		\begin{align}
			\label{dBestimation}
		\max_j\Vert\frac{\dx}{\dx z}B_j^{\lfloor\beta\rfloor,N}\Vert_{\mathcal{C}([0,1])}\leq (2N)^{\lfloor\beta\rfloor-2}\max_j\Vert B_j^{2,N}\Vert_{\mathcal{C}^1([0,1])}\leq 5\cdot2^{\lfloor\beta\rfloor-2}N^{\lfloor\beta\rfloor},
		\end{align}
		where in the last estimate we have used that
		\[
			\max_j\Vert B_j^{2,N}\Vert_{\mathcal{C}^1([0,1])}\leq 5N^2,
		\]
		which follows again from direct but tedious computations. By \eqref{betaestimation} we obtain
		\begin{align}
			\label{dBtildeestimation}
			\max_j\Vert\frac{\dx}{\dx z}\tilde{B}_j^{\lfloor\beta\rfloor,N}\Vert_{\mathcal{C}([0,1])}\leq 2^{\lfloor\beta\rfloor-1}CN^{\lfloor\beta\rfloor+1}M^{-K}+5\cdot2^{\lfloor\beta\rfloor-2}N^{\lfloor\beta\rfloor}.
		\end{align}
		using the triangle inequality. Now using \eqref{product_derivative}, we derive that
		\begin{multline*}
			\bigg\Vert\frac{\dx}{\dx z}\prod_{l=1}^dB_{j_l}^{\lfloor\beta\rfloor, N}-\frac{\dx}{\dx z}\prod_{l=1}^d\tilde{B}_{j_l}^{\lfloor\beta\rfloor, N,M}\bigg\Vert_{\mathcal{C}([0,1])}\\
			\leq \sum_{l=1}^d \bigg\Vert\frac{\dx}{\dx z} B_{j_l}^{\lfloor\beta\rfloor,N}\prod_{k\neq l}B_{j_k}^{\lfloor\beta\rfloor,N}-\frac{\dx}{\dx z}\tilde{B}_{j_l}^{\lfloor\beta\rfloor,N}\prod_{k\neq l}\tilde{B}_{j_k}^{\lfloor\beta\rfloor,N}\bigg\Vert_{\mathcal{C}([0,1])}.
		\end{multline*}
		Applying again a telescoping argument and \eqref{betaestimation}, \eqref{bestim}, \eqref{btildeestim}, \eqref{dBestimation}, \eqref{dBtildeestimation} we obtain  
		\begin{multline*}
			\bigg\Vert\frac{\dx}{\dx z} B_{j_l}^{\lfloor\beta\rfloor,N}\prod_{k\neq l}B_{j_k}^{\lfloor\beta\rfloor,N}-\frac{\dx}{\dx z}\tilde{B}_{j_l}^{\lfloor\beta\rfloor,N}\prod_{k\neq l}\tilde{B}_{j_k}^{\lfloor\beta\rfloor,N}\bigg\Vert_{\mathcal{C}([0,1])}\\
			\leq \tilde{C}N^{\lfloor\beta\rfloor+1}M^{-K}(N^{\lfloor\beta\rfloor+1}M^{-K}+N^{\lfloor\beta\rfloor})^{d-1}
		\end{multline*}
		for some $\tilde{C}>0$. Finally, combining this with \eqref{longestimate} yields that the term in \eqref{product_estim} fulfills the estimate
		\begin{align}
			\label{finalestimate}
		\bigg\Vert\prod_{l=1}^dB_{j_l}^{\lfloor\beta\rfloor, N}-\prod_{l=1}^d\tilde{B}_{j_l}^{\lfloor\beta\rfloor, N,M}\bigg\Vert_{\mathcal{C}^1([0,1])}\leq \tilde{C}N^{\lfloor\beta\rfloor+1}M^{-K}(N^{\lfloor\beta\rfloor+1}M^{-K}+N^{\lfloor\beta\rfloor})^{d-1}
		\end{align}
		for some $\tilde{C}>0$. This allows us to estimate the approximation error of the rational neural network $\mathcal{R}_f^{\lfloor\beta\rfloor,N,M}$ in terms of $N$ and $M$ since by the triangle inequality
		\begin{align*}
			\Vert f-\mathcal{R}_f^{\lfloor\beta\rfloor,N,M}\Vert_{\mathcal{C}^1([0,1]^d,\mathbb{R}^p)}\leq c N^{-(\beta-1)}+\max_{1\leq m\leq p}\Vert S_{f,m}^{\lfloor\beta\rfloor,N}-\mathcal{R}_{f,m}^{\lfloor\beta\rfloor,N,M}\Vert_{\mathcal{C}^1([0,1]^d)}
		\end{align*}
		using \eqref{Sfestimation}. Due to the representation of the tensor-product spline $S_{f,m}^{\lfloor\beta\rfloor,N}$ in \eqref{S_representation} and similarly, of the rational neural network $\mathcal{R}_{f,m}^{\lfloor\beta\rfloor,N,M}$, we obtain that
		\begin{multline*}
			\Vert S_{f,m}^{\lfloor\beta\rfloor,N}-\mathcal{R}_{f,m}^{\lfloor\beta\rfloor,N,M}\Vert_{\mathcal{C}^1([0,1]^d)}\\
			\leq \sum_{j_1,\dots,j_d=1}^{\lfloor\beta\rfloor+N}\vert w_{m,j_1,\dots,j_d}^{(f)}\vert\bigg\vert\prod_{l=1}^d(a_{j_l+\lfloor\beta\rfloor+1}-a_{j_l})\bigg\vert\bigg\Vert\prod_{l=1}^dB_{j_l}^{\lfloor\beta\rfloor, N}-\prod_{l=1}^d\tilde{B}_{j_l}^{\lfloor\beta\rfloor, N,M}\bigg\Vert_{\mathcal{C}^1([0,1])}.
		\end{multline*}
		Using the definition of $a_j$, the estimate
		\[
			\vert w_{m,j_1,\dots,j_d}^{(f)}\vert\leq C(f,\beta,d)
		\]
		from the proof of \cite[Theorem 2]{belomestny23}, and \eqref{finalestimate}, we obtain that 
		\begin{multline*}
			\Vert S_{f,m}^{\lfloor\beta\rfloor,N}-\mathcal{R}_{f,m}^{\lfloor\beta\rfloor,N,M}\Vert_{\mathcal{C}^1([0,1]^d)}\\
			\leq (\lfloor\beta\rfloor+N)^dC(f,\beta,d)\tilde{C}N^{\lfloor\beta\rfloor+1}M^{-K}(N^{\lfloor\beta\rfloor+1}M^{-K}+N^{\lfloor\beta\rfloor})^{d-1}.
		\end{multline*}
		Now choosing for fixed $\epsilon>0$, the degree parameter of the rational functions $M\approx N^\epsilon$, one can easily verify that for fixed $K\geq \epsilon^{-1}(\beta+d+d\lfloor\beta\rfloor)$ it holds 
		\[
			\Vert S_{f,m}^{\lfloor\beta\rfloor,N}-\mathcal{R}_{f,m}^{\lfloor\beta\rfloor,N,N^\epsilon}\Vert_{\mathcal{C}^1([0,1]^d)}\leq c N^{-(\beta-1)}
		\]
		for some $c>0$. With this, we recover the claimed approximation rate
		\begin{align}
			\label{final_estimation}
			\Vert f-\mathcal{R}_f^{\lfloor\beta\rfloor,N,N^\epsilon}\Vert_{\mathcal{C}^1([0,1]^d,\mathbb{R}^p)}\leq cN^{-(\beta-1)}
		\end{align}
		for some $c>0$, with $\mathcal{R}_f^{\lfloor\beta\rfloor,N,N^\epsilon}$ fulfilling the assertions of the theorem.
	\end{proof}
	The previous result immediately yields the following approximation result for rational functions.
	\begin{corollary}
		\label{cor:rat_approx}
		Let $\beta>2$ and $p,d\in\mathbb{N}$. Furthermore, let $f:[0,1]^d\to\mathbb{R}^p$ with $f\in\mathcal{H}_p^\beta([0,1]^d,H)$ for some $H>0$ and $\epsilon>0$. Then there exist a constant $c>0$ and rational functions $R^N$ of degree of order $N^{d+\epsilon}$ for $N\in\mathbb{N}$ such that
		\[
		\Vert f-R^N\Vert_{\mathcal{C}^1([0,1]^d,\mathbb{R}^p)}\leq cN^{-(\beta-1)}
		\]
	\end{corollary}
	\begin{proof}
		Following the proof of Theorem \ref{applic_theorem}, it holds true that $\tilde{B}_j^{2, N,M}$, which is the sum of maximally four rational functions each of degree of order $M$, is a rational function of degree of order $M$ for $1\leq j\leq 2\lfloor \beta\rfloor + N-2$ (see Remark \ref{rem:degree}). Employing the representation of the $B$-splines of order $2$ in the proof of Theorem \ref{applic_theorem} together with the recursive definition of the $\tilde{B}_j^{m, N,M}$ implies that also the $\tilde{B}_j^{m, N,M}$ for $3\leq m\leq\lfloor\beta\rfloor$ are maximally of degree of order $M$, holding in particular for $m=\lfloor\beta\rfloor$. As a consequence, we derive that the rational neural network	$\mathcal{R}_{f}^{\lfloor\beta\rfloor,N, M}$, which is defined in the proof of Theorem \ref{applic_theorem} and attains the representation
		\[
		\mathcal{R}_{f}^{\lfloor\beta\rfloor,N, M}(x)=\sum_{j_1,\dots,j_d=1}^{\lfloor\beta\rfloor+N}w_{m,j_1,\dots, j_d}^{(f)}\prod_{l=1}^d(a_{j_l+\lfloor\beta\rfloor+1}-a_{j_l})\tilde{B}_{j_l}^{\lfloor\beta\rfloor,N,M}(x_l),
		\]
		is a rational function of degree of order $d(N+\lfloor\beta\rfloor)^dM$, using Remark \ref{rem:degree}. Thus, the function $R^N:=\mathcal{R}_f^{\lfloor\beta\rfloor, N, N^\epsilon}$ is a rational function of degree of order $N^{d+\epsilon}$ and fulfills the claimed estimation due to \eqref{final_estimation}.
	\end{proof}
	\begin{remark}[Denominator polynomials]
		It can be easily shown using Remark \ref{remark:denominator} and following the proof of Theorem \ref{applic_theorem} that for $1\leq j_1,\dots, j_d\leq N+\lfloor\beta\rfloor$ the rational functions $[0,1]^d\ni x\mapsto \prod_{l=1}^d\tilde{B}_{j_l}^{\lfloor\beta\rfloor,N,M}(x_l)$ in Corollary \ref{cor:rat_approx} have denominator polynomials consisting only of monomial terms of even degree with positive coefficients, and hence also $R^N$ is of this form. As a consequence, the results in Theorem \ref{applic_theorem} and Corollary \ref{cor:rat_approx} show that already this restricted class of positive polynomials as denominators of rational functions provides quite strong approximation results.
	\end{remark}
	\begin{remark}[Different ReQU approximation]
		A crucial step in proving Theorem \ref{applic_theorem} is, given the design of the rational neural network, the rational approximation of the ReQU function, which is required for the approximation of the $B$-splines $B_j^{2,N}$. At this point, one is not restricted to the approximation result based on Newman polynomials presented in Corollary \ref{cor:requ}. For example, one could try to establish the approximation result using Zolotarev sign functions as in \cite{Boulle20} and take advantage of its representational benefits.
	\end{remark}
	\begin{remark}[Different architecture]
		It is important to note that in Theorem \ref{applic_theorem}, the $\mathcal{C}^1([0,1]^d)$-approximability follows from the $\mathcal{C}^1([0,1]^d)$-approximability result of the specific architecture in \cite{belomestny23} and $\mathcal{C}^1([-1,1])$-approximability of ReQU in Corollary \ref{cor:requ}. The main ingredients are thus, an architecture attaining a higher order universal approximation property and the approximation of its activation functions by rational functions in a norm that is at least as strong as the underlying higher order norm. Considering various higher-order approximability results, such as \cite{guehring20} for approximation in $W^{m,p}$-Sobolev norms for real $0\leq m\leq 1$ using the ReLU activation, or \cite{yang25} for Sobolev norms with integer $m\geq 2$ using both ReLU and ReQU activation, this suggests to study whether ReLU or ReQU activation can be approximated by rational functions in $W^{m,p}([-1,1])$ for respective values of $m$. This is, of course, a weaker requirement than the strong approximation result shown in Corollary \ref{cor:requ}.
	\end{remark}
	
	\subsection{Physical law learning via symbolic regression}
	\label{subsec:symbolic}
	In this section, we consider consequences of Corollary \ref{cor:rat_approx} for neural network based symbolic regression for learning physical laws. The latter, in very general terms, aims to determine a simply expressible symbolic function that fits given data coming from an input-output model. Here, we focus on techniques that use specific neural network architectures, as studied in \cite{scholl24} and \cite{sahoo18,lampert17} (see these references also for details on symbolic regression for equation learning). In short, these are (shallow or deep) feed forward neural networks with rational functions instead of affine linear transformations between the layers and general unary base functions (which cannot be expressed as rational functions and may be different for the different nodes) as activations. We first discuss the architectures \cite{scholl24} and \cite{sahoo18} and then consider a neural network design that covers the architectures in \cite{scholl24, sahoo18} as special cases. Note that the architecture in \cite{sahoo18} generalizes the one in \cite{lampert17}.
	
	The ParFam architecture proposed in \cite{scholl24} is a single hidden layer feed forward neural network with rational transformations and unary base activations as prescribed above. Additionally, the authors include a rational skip connection. Thus, by Corollary \ref{cor:rat_approx}, for increasing complexity of the rational degree of the skip connection, $\mathcal{C}^1$-approximability of suitably regular functions as in Corollary \ref{cor:rat_approx} is feasible.
	
	The $\text{EQL}^\div$ architecture proposed in \cite{sahoo18} is a deep feed forward neural network where the transformations between the layers are affine linear, the activations are \textit{well-behaving} unary base functions (functions such as $\exp$ and the square root are excluded) and multiplications and in the final layer a possible division is realized. The reason for these restrictive choices (compared to \cite{scholl24}) are of practical nature. 
	
	We consider the following architecture type, which covers the architectures in \cite{scholl24, sahoo18} as special cases, and show an approximation result of functions $f:[0,1]^d\to\mathbb{R}^p$ based on Corollary \ref{cor:rat_approx}. Let $d,L\in\mathbb{N}$ and consider vector valued $d$-ary base functions $\sigma_i:\mathbb{R}^d\to \mathbb{R}^d$ (which covers also $d$ unary base functions stacked in a $d$-dimensional array) for $1\leq i\leq d$. Furthermore, we assume $r_i:\mathbb{R}^d\to\mathbb{R}^d$ to be vector-valued rational functions of degree $N_i\in\mathbb{N}$, respectively, for $1\leq i\leq L$, where the coefficients of the numerator and denominator polynomials are trainable parameters. Finally, let $r_{L+1}:\mathbb{R}^d\to \mathbb{R}^p$ be similarly a vector-valued rational function of degree $N_{L+1}\in\mathbb{N}$. A general realization of the studied network is
	\begin{align}
		\label{symbolic_network}
		\mathfrak{S}_{r,\sigma}=r_{L+1}\circ\sigma_L\circ r_L\circ\dots\circ \sigma_1\circ r_1.
	\end{align}
	First, we formulate a pointwise $\mathcal{C}^1$-approximation result, essentially under regularity assumptions on the activation functions $(\sigma_i)_i$.
	\begin{lemma}
		\label{lem:symbolic}
		Let $\beta>2$ and $L,p,d\in\mathbb{N}$. Furthermore, let $f:[0,1]^d\to\mathbb{R}^p$ and $\sigma_i:\mathbb{R}^d\to\mathbb{R}^d$ locally Lipschitz continuous with $f\in\mathcal{H}^\beta_p([0,1]^d)$ and $\sigma_i^{-1}\in\mathcal{H}^\beta_d([0,1]^d)$ (with $\sigma^{-1}$ denoting the reverse function and being defined on $[0,1]^d$) for $1\leq i\leq L$. Then there exist rational functions $(r_i^n)_{n\in\mathbb{N}}$ for $1\leq i\leq L+1$ such that
		\begin{align}
			\label{pwconv}
			\lim_{n\to \infty}\mathfrak{S}_{r^n,\sigma}(x)=f(x)
		\end{align}
		for every $x\in[0,1]^d$. If $\nabla\sigma_i$ exists and is locally Lipschitz continuous then also
		\begin{align}
			\label{gradconv}
			\lim_{n\to\infty}\nabla\mathfrak{S}_{r^n,\sigma}(x)=\nabla f(x)
		\end{align}
		for every $x\in[0,1]^d$.
	\end{lemma}
	\begin{proof}
	There exist rational functions $(r_i^n)_{n\in\mathbb{N}}$ for $1\leq i\leq L+1$ with
		\begin{align}
			\label{rnconvergence}
			\lim_{n\to \infty}\Vert \sigma_i^{-1}-r_i^n\Vert_{\mathcal{C}^1([0,1]^d,\mathbb{R}^d)} = 0\quad \text{and}\quad \lim_{n\to \infty}\Vert f-r_{L+1}^n\Vert_{\mathcal{C}^1([0,1]^d,\mathbb{R}^p)} = 0
		\end{align}
		for $1\leq i\leq L$ by Corollary \ref{cor:rat_approx}. %
		We define for $1\leq i\leq L$ the networks
		\[
			\mathcal{S}_i^n = \sigma_i\circ r_i^n\circ\dots\circ\sigma_1\circ r_1^n
		\]
		and prove first via induction that for $1\leq i\leq L$,
		\begin{align}
			\label{claimSin}
		\lim_{n\to \infty}\mathcal{S}_i^n(x)=x\quad\text{and}\quad\lim_{n\to\infty}\nabla\mathcal{S}_i^n(x)=I
		\end{align}
		with $I\in\mathbb{R}^{d\times d}$ the identity matrix. For $i=1$ by local Lipschitz continuity of $\sigma_1$ there exists $\tilde{\delta}>0$ and $L_{\sigma_1^{-1}(x)}^{\sigma_1}>0$ such that $\sigma_1$ is Lipschitz continuous in $\mathcal{B}_{\tilde{\delta}}(\sigma_1^{-1}(x))$ with Lipschitz constant $L_{\sigma_1^{-1}(x)}^{\sigma_1}$. By \eqref{rnconvergence} for $0<\delta<\tilde{\delta}$ and sufficiently large $n$ it holds $\Vert \sigma_1^{-1}-r_1^n\Vert_{\mathcal{C}^1([0,1]^d,\mathbb{R}^d)}<\delta$. As a consequence, we derive the estimate
		\[
			\vert \mathcal{S}_1^n(x)-x\vert=\vert \sigma_1(r_1^n(x))-\sigma_1(\sigma_1^{-1}(x))\vert\leq L_{\sigma_1^{-1}(x)}^{\sigma_1}\vert r_1^n(x)-\sigma_1^{-1}(x)\vert < L_{\sigma_1^{-1}(x)}^{\sigma_1}\delta,
		\]
		which can be made arbitrarily small by the choice of $\delta$ and shows the first part in \eqref{claimSin} for $i=1$. The second part follows under the additional assumption of local Lipschitz continuity of $\nabla\sigma_1$. For that, let $\tilde{\delta}$ be small enough such that $\nabla\sigma_1$  is Lipschitz continuous in $\mathcal{B}_{\tilde{\delta}}(\sigma_1^{-1}(x))$ with Lipschitz constant $L_{\sigma_1^{-1}(x)}^{\nabla\sigma_1}$. Then inserting a productive zero and using $I=\nabla[(\sigma_1\circ\sigma_1^{-1})(x)]=\nabla\sigma_1(\sigma_1^{-1}(x))\nabla\sigma_1^{-1}(x)$ together with the triangle inequality implies that
		\begin{align*}
			\vert \nabla\mathcal{S}_1^n(x)-I\vert&\leq \vert\nabla\sigma_1(r_1^n(x))-\nabla\sigma_1(\sigma_1^{-1}(x))\vert\vert\nabla r_1^n(x)\vert\\
			&\qquad +\vert\nabla\sigma_1(\sigma_1^{-1}(x))\vert\vert\nabla r_1^n(x)-\nabla\sigma_1^{-1}(x)\vert\\
			&\leq L_{\sigma_1^{-1}(x)}^{\nabla\sigma_1}\vert r_1^n(x)-\sigma_1^{-1}(x)\vert(\delta+\vert\nabla\sigma_1^{-1}(x)\vert)+\vert\nabla\sigma_1(\sigma_1^{-1}(x))\vert\delta\\
			&\leq L_{\sigma_1^{-1}(x)}^{\nabla\sigma_1}\delta(\delta+\vert\nabla\sigma_1^{-1}(x)\vert)+\vert\nabla\sigma_1(\sigma_1^{-1}(x))\vert\delta,
		\end{align*}
		which can be made arbitrarily small again by choosing $\delta$ small.
		
		For the induction step assume that \eqref{claimSin} holds for some $i\leq L-1$. By local Lipschitz continuity of $\sigma_{i+1}^{-1}$ (it is even assumed to be higher order Hölder continuous) there exists $\epsilon>0$ and $L^{\sigma_{i+1}^{-1}}_x>0$ such that $\sigma_{i+1}^{-1}$ is Lipschitz continuous in $\mathcal{B}_\epsilon(x)$ with Lipschitz constant $L_x^{\sigma_{i+1}^{-1}}$. By induction hypothesis for sufficiently large $n$ it holds true that $\vert \mathcal{S}_i^n(x)-x\vert<\epsilon$. By \eqref{rnconvergence} for any $\delta>0$ we have that $\Vert \sigma_{i+1}^{-1}-r_{i+1}^n\Vert_{\mathcal{C}^1([0,1]^d,\mathbb{R}^d)}<\delta$ for sufficiently large $n$. Thus, we obtain
		\[
			\vert r_{i+1}^n(\mathcal{S}_i^n(x))-\sigma_{i+1}^{-1}(x)\vert\leq \vert r_{i+1}^n(\mathcal{S}_i^n(x))-\sigma_{i+1}^{-1}(\mathcal{S}_i^n(x))\vert+\vert \sigma_{i+1}^{-1}(\mathcal{S}_i^n(x))-\sigma_{i+1}^{-1}(x)\vert
		\]
		and hence, using the previous estimate for sufficiently large $n$ that
		\[
			\vert r_{i+1}^n(\mathcal{S}_i^n(x))-\sigma_{i+1}^{-1}(x)\vert\leq \delta+\epsilon L_x^{\sigma_{i+1}^{-1}}.
		\]
		By local Lipschitz continuity of $\sigma_{i+1}$ there exists $\tilde{\delta}>0$ and $L_{\sigma_{i+1}^{-1}(x)}^{\sigma_{i+1}}>0$ such that $\sigma_{i+1}$ is Lipschitz continuous in $\mathcal{B}_{\tilde{\delta}}(\sigma_{i+1}^{-1}(x))$ with Lipschitz constant $L_{\sigma_{i+1}^{-1}(x)}^{\sigma_{i+1}}$.
		As a consequence, we infer for small $\delta, \epsilon$ with $\delta+\epsilon L_x^{\sigma_{i+1}^{-1}}<\tilde{\delta}$ (which, in particular implies by the above considerations that $r_{i+1}^n(\mathcal{S}_i^n(x)), \sigma_{i+1}^{-1}(\mathcal{S}_i^n(x))\in \mathcal{B}_{\tilde{\delta}}(\sigma_{i+1}^{-1}(x))$) that
		\begin{align*}
			\vert \mathcal{S}_{i+1}^n(x)-x\vert &\leq \vert \mathcal{S}_{i+1}^n(x)- \mathcal{S}_{i}^n(x)\vert +\vert\mathcal{S}_{i}^n(x)-x\vert\\
			&\leq \vert \sigma_{i+1}(r_{i+1}^n(\mathcal{S}_i^n(x)))-\sigma_{i+1}(\sigma_{i+1}^{-1}(\mathcal{S}_i^n(x)))\vert+\epsilon\\
			&\leq L_{\sigma_{i+1}^{-1}(x)}^{\sigma_{i+1}}(\delta+\epsilon L_x^{\sigma_{i+1}^{-1}})+\epsilon,
		\end{align*}
		which can be made arbitrarily small by choosing $\delta, \epsilon$ sufficiently small, showing the first part in \eqref{claimSin}. The second part follows under the additional assumption of local Lipschitz continuity of $\nabla\sigma_{i+1}$. For that, let $\tilde{\delta}$ be small enough such that $\nabla\sigma_{i+1}$  is Lipschitz continuous in $\mathcal{B}_{\tilde{\delta}}(\sigma_{i+1}^{-1}(x))$ with Lipschitz constant $L_{\sigma_{i+1}^{-1}(x)}^{\nabla\sigma_{i+1}}$. As
		\[
			\nabla\mathcal{S}_{i+1}^n(x)=\nabla\sigma_{i+1}(r_{i+1}^n(\mathcal{S}_i^n(x)))\nabla r_{i+1}^n(\mathcal{S}_i^n(x))\nabla\mathcal{S}_i^n(x)
		\]
		by the chain rule and $\lim_{n\to\infty}\nabla\mathcal{S}_i^n(x)=I$ by assumption it suffices to show that
		\[
			\lim_{n\to \infty}\nabla\sigma_{i+1}(r_{i+1}^n(\mathcal{S}_i^n(x)))\nabla r_{i+1}^n(\mathcal{S}_i^n(x))=I.
		\]
		Indeed, we can estimate similarly as in the induction start,
		\begin{multline*}
			\vert \nabla\sigma_{i+1}(r_{i+1}^n(\mathcal{S}_i^n(x)))\nabla r_{i+1}^n(\mathcal{S}_i^n(x))-I\vert\\
			\leq \vert\nabla\sigma_{i+1}(r_{i+1}^n(\mathcal{S}_i^n(x)))-\nabla\sigma_{i+1}(\sigma_{i+1}^{-1}(\mathcal{S}_i^n(x)))\vert\vert\nabla r_{i+1}^n(\mathcal{S}_i^n(x))\vert\\
			+\vert\nabla\sigma_{i+1}(\sigma_{i+1}^{-1}(\mathcal{S}_i^n(x)))\vert\vert\nabla r_{i+1}^n(\mathcal{S}_i^n(x))-\nabla\sigma_{i+1}^{-1}(\mathcal{S}_i^n(x))\vert,
		\end{multline*}
		which can be further estimated by
		\[
			\delta L_{\sigma_{i+1}^{-1}(x)}^{\nabla\sigma_{i+1}}(\vert\nabla\sigma_{i+1}^{-1}(\mathcal{S}_i^n(x))\vert+\delta)+\vert\nabla\sigma_{i+1}(\sigma_{i+1}^{-1}(\mathcal{S}_i^n(x)))\vert\delta.
		\]
		This upper bound can be made arbitrarily small by choice of $\delta$ as the terms $\vert\nabla\sigma_{i+1}^{-1}(\mathcal{S}_i^n(x))\vert$ and $\vert\nabla\sigma_{i+1}(\sigma_{i+1}^{-1}(\mathcal{S}_i^n(x)))\vert$ are bounded for sufficiently large $n$ due to $\vert \mathcal{S}_i^n(x)-x\vert<\epsilon$, and local Lipschitz continuity of $\nabla\sigma_{i+1}$, $\sigma_{i+1}^{-1}$ and $\nabla\sigma_{i+1}^{-1}$ (the latter is assumed to be even higher order Hölder continuous).\\
		
		This implies in particular that
		\begin{align}
			\label{SLpwconv}
			\lim_{n\to \infty}\mathcal{S}_L^n(x)=x
		\end{align}
		for $x\in[0,1]^d$ and additionally, under local Lipschitz continuity of the $\nabla\sigma_i$ also
		\begin{align}
			\label{SLgradconv}
		\lim_{n\to \infty}\nabla\mathcal{S}_L^n(x)=I.
		\end{align}
		As $\mathfrak{S}_{r^n,\sigma}(x)=r_{L+1}^n(\mathcal{S}_L^n(x))$ the claim in \eqref{pwconv} follows using \eqref{rnconvergence}, \eqref{SLpwconv} and continuity of $f$. Similarly, the claim in \eqref{gradconv} follows using \eqref{rnconvergence}, \eqref{SLgradconv}, continuity of $\nabla f$ and
		\begin{align*}
			\vert \nabla \mathfrak{S}_{r^n,\sigma}(x)-\nabla f(x)\vert &= \vert\nabla r_{L+1}^n(\mathcal{S}_L^n(x))\nabla\mathcal{S}_L^n(x)-\nabla f(x)\vert\\
			&\leq\vert \nabla r_{L+1}^n(\mathcal{S}_L^n(x))-\nabla f(\mathcal{S}_L^n(x))\vert\vert\nabla\mathcal{S}_L^n(x)\vert\\
			&\quad +\vert\nabla f(\mathcal{S}_L^n(x))
-\nabla f(x)\vert\vert\nabla \mathcal{S}_L^n(x)\vert\\
			&\quad +\vert \nabla \mathcal{S}_L^n(x)-I\vert\vert\nabla f(x)\vert.
		\end{align*}
	\end{proof}
	\begin{remark}[Examples of activations]
		In the previous result, we require that the activations $\sigma_i:\mathbb{R}^d\to\mathbb{R}^d$ are locally Lipschitz continuous and satisfy $\sigma_i^{-1}\in\mathcal{H}^\beta_d([0,1]^d)$ for $1\leq i\leq L$. In view of the necessity of the latter condition, which is generally restrictive, we refer to the remark below. Examples for $d=1$ include $\sigma(x)=\exp(x)-1$ and $\sigma(x)=2\arctan(x)$, which fulfill $\sigma^{-1}\in\mathcal{H}^\beta_1([0,1])$ for $\beta>2$.
	\end{remark}
	\begin{corollary}
		The result of Lemma \ref{lem:symbolic} holds uniformly in \eqref{pwconv} and \eqref{gradconv} if the local Lipschitz continuity assumption is replaced by Lipschitz continuity, respectively.
	\end{corollary}
	\begin{proof}
		The assertion can be shown following the proof of Lemma \ref{lem:symbolic} and simplifying arguments along the lines based on global Lipschitz continuity. Furthermore, one has to use that bounded sets under globally continuous functions (i.e. on entire $\mathbb{R}^d$) are bounded and the preimage of a compact set under a continuous function is compact.
	\end{proof}
	\begin{remark}
		What is unsatisfactory about the proof of Lemma \ref{lem:symbolic} is its nature of attempting to cancel the effect of the activation functions and essentially using one rational transformation to approximate the target function $f$ accordingly. Approaches that exploit the non-rationality of the activation functions would be interesting in this context, particularly in order to avoid potentially strong regularity assumptions about the inverse functions of the activations. Nevertheless, it is important to note that, from a practical point of view, an optimally trained network of the above type performs at least as well as the rates proven in the previous results.
	\end{remark}

	\section{Conclusions}
	In this work, we have considered $\mathcal{C}^1$ approximability of functions by rational neural networks, using a spline-based ReQU neural network architecture that attains a higher order universal approximation property. This result is further transferred to an analogous approximation result for rational functions. Furthermore, an application to neural network based symbolic regression for physical law learning is considered.
	
	Future research directions comprise investigating better first order approximation power under a lower degree of the rational functions and $\mathcal{C}^ k$-approximability with $k \geq 2$.
	
	\bibliographystyle{plain}
	\bibliography{references}
\end{document}